\documentclass[lettersize,journal]{IEEEtran}
\usepackage{amsmath,amsfonts}
\usepackage{algorithmic}
\usepackage{algorithm}
\usepackage{array}
\usepackage[caption=false]{subfig}
\usepackage{textcomp}
\usepackage{stfloats}
\usepackage{url}
\usepackage{verbatim}
\usepackage{graphicx}
\usepackage{cite}
\usepackage{hyperref}

\usepackage{cleveref}
\usepackage{amsthm}

\newtheorem{claim}{Claim}
\newtheorem{remark}{Remark}

\newtheorem{prerequisite}{Prerequisite}

\begin{document}

\title{Feasibility of Local Trajectory Planning for Level-2+ Semi-autonomous Driving without Absolute Localization}

\author{Sheng Zhu$^\ast$, Jiawei Wang$^\dagger$, Yu Yang$^\dagger$, Bilin Aksun-Guvenc$^\ast$, ~\IEEEmembership{Member,~IEEE}
\thanks{$\ast$ Sheng Zhu and Bilin Aksun-Guvenc are with the Automated Driving Lab, Ohio State University, OH 43212, USA. email: {zhu.1473, aksunguvenc1.}@osu.edu, Cooresponding author: zhu.1473@osu.edu}
\thanks{$\dagger$ Jiawei Wang and Yu Yang are with the State Key Laboratory of Automotive Simulation and Control, Jilin University, Changchun 130022, China. email: wangjw17@mails.jlu.edu.cn, yyu0821@163.com.}
\thanks{Manuscript received Sept XX, 2023; revised XX, 2023.}}

\markboth{Journal of \LaTeX\ Class Files,~Vol.~XX, No.~X, September~2023}%
{Shell \MakeLowercase{\textit{et al.}}: A Sample Article Using IEEEtran.cls for IEEE Journals}


\maketitle

\begin{abstract}

Autonomous driving has long grappled with the need for precise absolute localization, making full autonomy elusive and raising the capital entry barriers for startups. This study delves into the feasibility of local trajectory planning for level-2+ (L2+) semi-autonomous vehicles without the dependence on accurate absolute localization. Instead, we emphasize the estimation of the pose change between consecutive planning frames from motion sensors and integration of relative locations of traffic objects to the local planning problem under the ego car's local coordinate system, therefore eliminating the need for an absolute localization. Without the availability of absolute localization for correction, the measurement errors of speed and yaw rate greatly affect the estimation accuracy of the relative pose change between frames. We proved that the feasibility/stability of the continuous planning problem under such motion sensor errors can be guaranteed at certain defined conditions. This was achieved by formulating it as a Lyapunov-stability analysis problem. Moreover, a simulation pipeline was developed to further validate the proposed local planning method. Simulations were conducted at two traffic scenes with different error settings for speed and yaw rate measurements. The results substantiate the proposed framework's functionality even under relatively inferior sensor errors. We also experiment the stability limits of the planned results under abnormally larger motion sensor errors. The results provide a good match to the previous theoretical analysis. Our findings suggested that precise absolute localization may not be the sole path to achieving reliable trajectory planning, eliminating the necessity for high-accuracy dual-antenna GPS as well as the high-fidelity maps for SLAM localization.


\end{abstract}


\section{Introduction}
\IEEEPARstart{O}{ver} the years, fully autonomous driving has struggled to achieve reliability and large-scale implementation. Even today, leading representatives in autonomous driving, such as Waymo and Cruise, face challenges navigating urban environments like San Francisco, with their driverless vehicles occasionally clogging traffic in the middle of the road \cite{wired2023, nytimes2023}. Despite more than a decade of industry R\&D, driverless autonomous solutions appear far from successful business applications based on current performance. Meanwhile, autonomous driving startups are rapidly consuming investments. Uber alone reportedly spent an annual \$457 million on self-driving R\&D before selling its unit to Aurora \cite{tcrunch2019}. Over time, investors have become less attracted to driverless technology and increasingly hesitant about the technology's potential business yields in the near future. The landmark shutdown of the star unicorn startup Argo AI \cite{tcrunch2022} and the closure of the first public self-driving truck company Embark Technology \cite{atnews2023} both reflect the struggles of Level-4 (L4) startups to secure capital funding. The heightened global economic uncertainty, along with potential recession, exacerbates this struggle. L4 providers like Waymo and TuSimple have reportedly laid off large numbers of employees to cut operating costs \cite{cnn2023,reuters2022}.

The valuation of autonomous driving startups has plummeted drastically in recent times. Waymo's valuation dropped 80\% from \$200 billion to \$30 billion in just 18 months. TuSimple's stock, which once peaked above \$60, now stands at slightly above \$1 as of April 2023. In contrast, Tesla's Autopilot driver assistance system, first deployed in 2014, has been a significant feature that greatly promotes the sale of Tesla vehicles. Its so-called Full Self-Driving (FSD) add-on package, though arguably named, directly contributed \$324 million in revenue for the fourth quarter alone in 2022 \cite{cnbc2023}. The FSD package is so popular that 19\% of Tesla owners opted in despite the price hike from an initial \$5,000 in 2019 to \$15,000 in 2022 \cite{iev2023}. The debate over the autonomous driving development strategy between Tesla and Waymo has persisted for years \cite{mday2019}. Tesla insisted on a progressive evolution path with a vision-based solution that emphasizes cost reduction and mass production application, while Waymo aimed for an all-in-one driverless solution from day one, incorporating advanced and expensive sensors like Lidar and relying heavily on detailed prerequisite information, such as high-definition (HD) maps. As of now, it seems that the Tesla approach prevails in the industry, generating continuous cash flow to back support its own progression.
\IEEEpubidadjcol

The autonomous vehicle industry has quietly but largely shifted interest towards cost-efficient Level-2+ (L2+) semi-autonomous driving solutions. OEMs are especially interested in the potential urban Navigate on Autopilot (NOA) feature \cite{rc2022}, an L2+ feature that traditional Tier-1 supplier unable to provide, to assist driver navigate in complex urban scenarios. L2+ solutions do not guarantee driving safety and demand human attention and intervention during driving. In such applications, "The driver is still responsible for, and ultimately in control of, the car," as Tesla stated \cite{nyu2022}. This easing of liability and the compromise on fully self-driving realization provide L4 startups a midway transition to package their solutions into a viable product. However, this transition is not merely a hardware downgrade and algorithm transplant. The lite version of hardware may imply fewer available resources, such as weaker computing power, unavailability of sensor data, or less accurate measurement data, among others. As a result, transplanting L4 algorithms, which are designed for data-abundant hardware platforms, to fit L2+ applications is not as simple as it may seem \cite{eo2022}.

One challenge to deal with is the lack of highly-accurate global/absolute localization. Centimeter-level accuracy of absolute localization is essential for L4 driving request detailed lane information from HD map. This high localization accuracy is usually achieved by the fusion of GNSS/INS (Inertial Navigation System), or simultaneous localization and mapping (SLAM) in areas with weak GPS signals by matching Lidar data and the normal distribution transform (NDT) map. However, in cost-sensitive mass production, L2+ vehicles do not come equipped with expensive dual-antenna GPS. NDT maps are also not available to cover every traffic route to work with SLAM. Therefore, for L2+ applications, centimeter-level accuracy of absolute localization is not available. Now the question arises: without this critical absolute localization information, is semi-autonomous driving still technically feasible? 

Consider the case of human driving: we don't necessarily need to be aware of our absolute localization by centimeter accuracy. One also does not have the detailed lane-to-lane transition routes at intersections, as HD maps provide. We are more aware of our surroundings, such as the distance from other traffic objects or whether we are in the correct lane. This analogy to human driving may seem simplistic but implies that L2+ semi-autonomous driving may be feasible without accurate absolute localization, but instead using relative localization. Relative localization refers to the process of determining the relative position and orientation of the ego vehicle with respective to the surrounding environment, including other vehicles, pedestrians, and obstacles. The relative localization is typically achieved from perception system using a combination of sensors such as cameras, radars, and possibly Llidars, which are available for L2+ ready vehicles. 

But how does local trajectory planning works with relative localization? During planning, one consideration is the trajectory consistency, meaning the trajectories planned in consecutive frames must maintain or approximate the spatial and temporal continuity. This is achievable for absolute localization since the ego vehicle localization, planned trajectories, and traffic object movements are all under the same global coordinate system. However, for relative localization, there's no way to accurately reflect all these information under the global coordinate system. Few research have been addressed this practical problem. To ensure the consistency of the planned trajectory, we emphasized that relative localization with respect to surrounding traffic between adjacent frames must be associated. The relative motion of the ego vehicle between adjacent frames can be estimated by integrating acceleration and yaw rate data from the inertial measurement unit (IMU) sensor. This integration inevitably introduces error in the estimation of position and posture changes. The INS is based on exactly this idea to estimate absolute localization but needs periodic correction/fusion with GNSS data to avoid the build-up of the integration errors. In the L2+ semi-autonomous driving case, there is no accurate source to correct the INS estimate of absolute localization. Therefore, instead of choosing global coordinates, trajectory planning for L2+ semi-autonomous driving was done under the local vehicle coordinate. The trajectory from last frame was projected to the local vehicle coordinate at current frame to ensure planning consistency. In this way, the localization error is limited between frames and will not accumulate. 

Although the trajectory planning topic is not new in research fields with different approaches raised based on spline\cite{zhu2018online,guvenc2021autonomous}, potential fields\cite{ji2023tripfield}, sampling method \cite{ma2015rrt, zhu2020trajectory}, graph search \cite{dolgov2010path}, optimization \cite{chen2017constrained}, and so on, few have evaluated the dynamic stability of the continuously changing planned trajectories. Trajectory planning is usually continuously ongoing in cycles and may subject to change in response to perturbations including changes in the environment, sensor noise, execution errors from control, and model uncertainties from vehicle dynamics and the environment. It's important to ensure the planned trajectory is stable and feasible under such perturbations. \cite{Tsukamoto2021} presented a learning-based motion planning with stability guaranteed by designing a differential Lyapunov function using contraction theory. In \cite{englot2009perception} a motion planning framework was designed to maximize a marine vehicle's stability margins against ocean disturbances. But few comparable analysis have been seen for trajectory planning in autonomous driving field. \cite{gu2015} demonstrated the stability of the cost-based lattice controller in event of dynamic environment change but limited to simulation without rigorous theoretical proof. Under the context of trajectory planning without absolute localization, the drift and offset errors from the IMU sensor could build up the estimated relative localization error and affects the continuously planning stability. Hence this paper specifically considers perturbations from relative localization and proves the necessary conditions under which the trajectory planning could maintain dynamic stability.

One major contribution from this paper is the proposal of local trajectory planning framework that works without the availability of absolute localization, which is challenging for L2+ semi-autonomous driving applications with limited hardware. Another contribution comes from the proof and conclusion that stability of the dynamic trajectory planning subject to motions sensor errors could be ensured under certain conditions easy to met. The stability of the proposed local trajectory planning framework is also further validated under different simulation scenarios given drifting and offset noise from the IMU sensor. The paper could bring some insights and prove the validity for L2+ semi-autonomous application development with limited hardware equipment. 

Organization of the paper is in follows: In Section II, the methodology of L2+ local trajectory planning without absolute localization was introduced. This is followed by proof and discussion on the effects of the relative localization error on the trajectory planning stability in Section \ref{sec:feasibility}. A simulation pipeline was built based on the proposed L2+ trajectory planning framework. The effects of drifting and offset noise from the IMU sensors on the continuous trajectory planning result were shown in simulation results and discussed in Section IV. Based on the theoretical analysis as well as simulation results under measurement errors of speed and yaw rate, we concluded that local trajectory planning for semi-autonomous driving without absolute localization is feasible. 

\section{Methodology}
\subsection{Trajectory Planning Framework}
Without loss of generality, assume that the vehicle state $x$ can be represented by projection of its position and heading in the XY plane for simplicity, $x \in \mathbb{R}^3$. 
The trajectory planning problem for the semi-autonomous driving vehicle is to determine its trajectory as a function of time to avoid collisions with obstacles or intrusions to untraversable area as by traffic rules. 

Let $T=[t_0, t_f]$ represent the time interval over which the trajectory of the semi-autonomous vehicle is to be planned, where $t_0$ and $t_f$ denote the planning initial time and terminal time respectively. The vehicle position is represented as $x \in \mathbb{R}^3$, with $x_0$ and $x_f$ indicating the position at time $t_0$ and $t_f$ separately. Let $O$ denote the set of road objects (obstacles, other road users, illegal traffic area) that are not traversable. $O(t) = \left \{ O_1(t) \cup O_2(t) \cup \dots \cup O_n(t) \right \} \subset \mathbb{R}^3$ is the space occupied by all road objects at time $t$. The vehicle trajectory can be interpreted as the continuous mapping $\mathcal{T}$ from $T$ to $\mathbb{R}^3$ that does not overlap with $O(t)$. Note that the map $\mathcal{T}$ has to be continuous to be physically realizable. The trajectory planning problem can be formally stated as \cite{kant1986toward}: 

\begin{list}{}{}
\item{ Find a mapping $\mathcal{T}: T \to \mathbb{R}^3$ with $x(t_0) = x_0$, $x(t_f) = x_f$, such that $\forall t \in T$, $x(t )|_{  \mathcal{T}}\notin O(t)$. }  
\end{list}

During driving, the complex and dynamic-changing traffic environment demands the planning to continuously update the trajectory. To ensure the continuity of the planned trajectories between frames, planning at current time step/frame usually sets the initial start point $x_0$ from last planning result. This would also have the benefit to decouple the planning process from the execution result of its downstream control module. The diagram in Figure \ref{fig:fig_plan} shows the implementation of the trajectory planning proposed for level-2+ semi-autonomous driving. 

\begin{figure*}[!t]
\centering
\includegraphics[width=\textwidth]{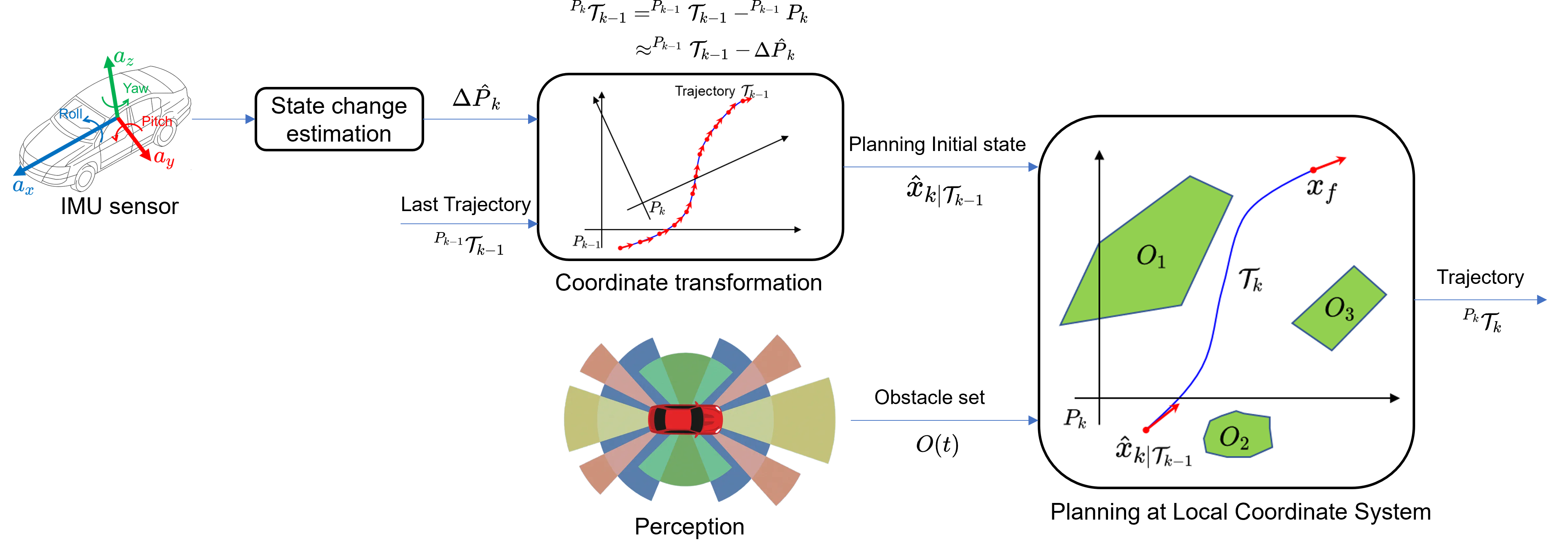}
\caption{Proposed local trajectory planning for level-2+ semi-autonomous driving without absolute localization}
\label{fig:fig_plan}
\end{figure*}

As shown in Figure \ref{fig:fig_plan}, at frame/timestep $t_k$ (note that frame and time step may be used interchangeably in this paper), the new trajectory $\mathcal{T}_k$ is planned under the local vehicle coordinate system $P_k$ from a planning initial state given the obstacle set $O(t)$ from the perception results. The perception result $O(t)$ itself is with respect to local vehicle coordinate system $P_k$ and does not need coordinate transformation. The planning initial state $x_k|\mathcal{T}_{k-1}$ represents the planned-ahead state for time $t_{k}$ by then last trajectory $/mathcal{T_{k-1}}$. Its projection to current local coordinate system $P_k$ however requires coordinate transformation: 

\begin{align}
\label{eq:coord}
  ^{P_k}\mathcal{T}_k &= ^{P_{k-1}}\mathcal{T}_{k-1} - ^{P_{k-1}}{P}_{k} \\
    &\approx  ^{P_{k-1}}\mathcal{T}_{k-1} - \Delta \hat P_k,
\end{align}
where $\Delta \hat P _k$ is the estimation for state change $^{P_{k-1}}{P}_{k}$. $\Delta \hat P _k$ could be derived from the onboard IMU sensor and/or wheel speed sensors: 

\begin{equation}
\label{eq:state_est1}
   \Delta \hat P_k = \begin{pmatrix} 
    \Delta \hat x \\
    \Delta \hat y \\
    \Delta \hat \theta
    \end{pmatrix} = \begin{bmatrix}
    \iint\limits_{t_{k-1}}^{t_k} \left ( a_x \cos{\Delta\hat\theta} - a_y \sin{\Delta\hat\theta}\right )  dt^2 \\
    \iint\limits_{t_{k-1}}^{t_k} \left ( a_x \sin{\Delta\hat\theta} + a_y \cos{\Delta\hat\theta}\right )  dt^2 \\
    \int\limits_{t_{k-1}}^{t_k}  \dot{\theta} dt
    \end{bmatrix}, 
\end{equation}
or more simply
\begin{equation}
\label{eq:state_est2}
   \Delta \hat P_k = \begin{pmatrix} 
\Delta \hat x \\
\Delta \hat y \\
\Delta \hat \theta
\end{pmatrix} = \begin{bmatrix}
\int\limits_{t_{k-1}}^{t_k} v\cos{\Delta \hat\theta}dt \\
\int\limits_{t_{k-1}}^{t_k} v\sin{\Delta \hat \theta}dt \\
\int\limits_{t_{k-1}}^{t_k}  \dot{\theta} dt
\end{bmatrix},
\end{equation}
where $v$ is the deduced speed from wheel speed sensors and vehicle lateral speed is ignored. 

Compared with level-4 planning, because of the unavailability of absolute localization information, the proposed planning for level-2+ semi-autonomous driving is done under the local vehicle coordinate system, which makes the relative state change estimation and coordinate transformation process necessary to associate planning result from last frame to current frame. 

\subsection{Validation Pipeline} \label{subsec_valid_pip}
In this work, a validation pipeline is developed to further validate the proposed local trajectory planning methodology without absolute localization in simulation. At each frame, the trajectory planning is done using the sampling-based method proposed in \cite{werling2012optimal}. In \cite{werling2012optimal}, a series of quintic polynomials are generated in lateral and longitudinal directions respectively under the Frenet coordinate system, and are then combined to form a pool of candidate trajectories. The "best" collision-free trajectory is then selected from this pool from a defined cost function which considers driving comfort and safety with respect to road objects. 

To simplify the validation pipeline, the control module is not included and hence the tracking errors are out of the discussion. But as we mentioned earlier, the proposed local planning method is decoupled from the downstream control module anyway so this simplification will not affect our final results. Instead, we assume that the trajectory is executed by the timestamp during trajectory execution. 

The estimation error of the state change is partly because of measurement errors from the sensors. We assume the following measurement models for the vehicle speed from wheel speed sensors, and yaw rate from the yaw rate sensor: 
\begin{equation}
\label{eq:speed_error}
    v_m = v_{\text{offset}} + \tilde{v}, \text{ where } \tilde{v}  \sim \mathcal{N}(v, \sigma_v^2). 
\end{equation}

\begin{equation}
\label{eq:yaw_rate_error}
    \dot{\theta}_m = \dot{\theta}_\text{offset} + \tilde{\dot{\theta}}, \text{ where } \tilde{\dot{\theta}}  \sim \mathcal{N}(\dot{\theta}, \sigma_{\dot{\theta}}^2). 
\end{equation}

\section{Feasibility/Stability Analysis}\label{sec:feasibility}

\subsection{Problem Description}

Equation \eqref{eq:state_est1} and \eqref{eq:state_est2} gives an approximation of relative motion change between consecutive planning frames. However, due to the integrals in the equations, the estimation of vehicle position and posture changes since the first frame could build up as time goes. The accumulated error could have adverse effects on the proposed local trajectory planning method, making it infeasible to achieve to the desired destination. 

 To simply the analysis, the terminal destination $x_f$ of the trajectory planning for consecutive frames is assumed unchanged in order to show the stability concept for continuous planning. This simplification is legit for consecutive frames in scenarios like traffic stop, lane changing, and etc. 
 
 The diagram in Figure \ref{fig:err} illustrates the potential effect of the accumulated estimation error to the continuous trajectory planning results at timestep $t_{k-1}$ and its next timestep $t_k$. To clearly show the drift of the trajectory due to the estimation error, the trajectories $\mathcal T_{k-1}$ and $\mathcal T_k$ planned under each own local coordinate system are put under the same global coordinate system, under which the terminal destination is set to $\bf 0$ without loss of generality. 

\begin{figure*}[!t]
\centering
\subfloat[Estimation error $\varepsilon=0$]{\includegraphics[height=3in,keepaspectratio]{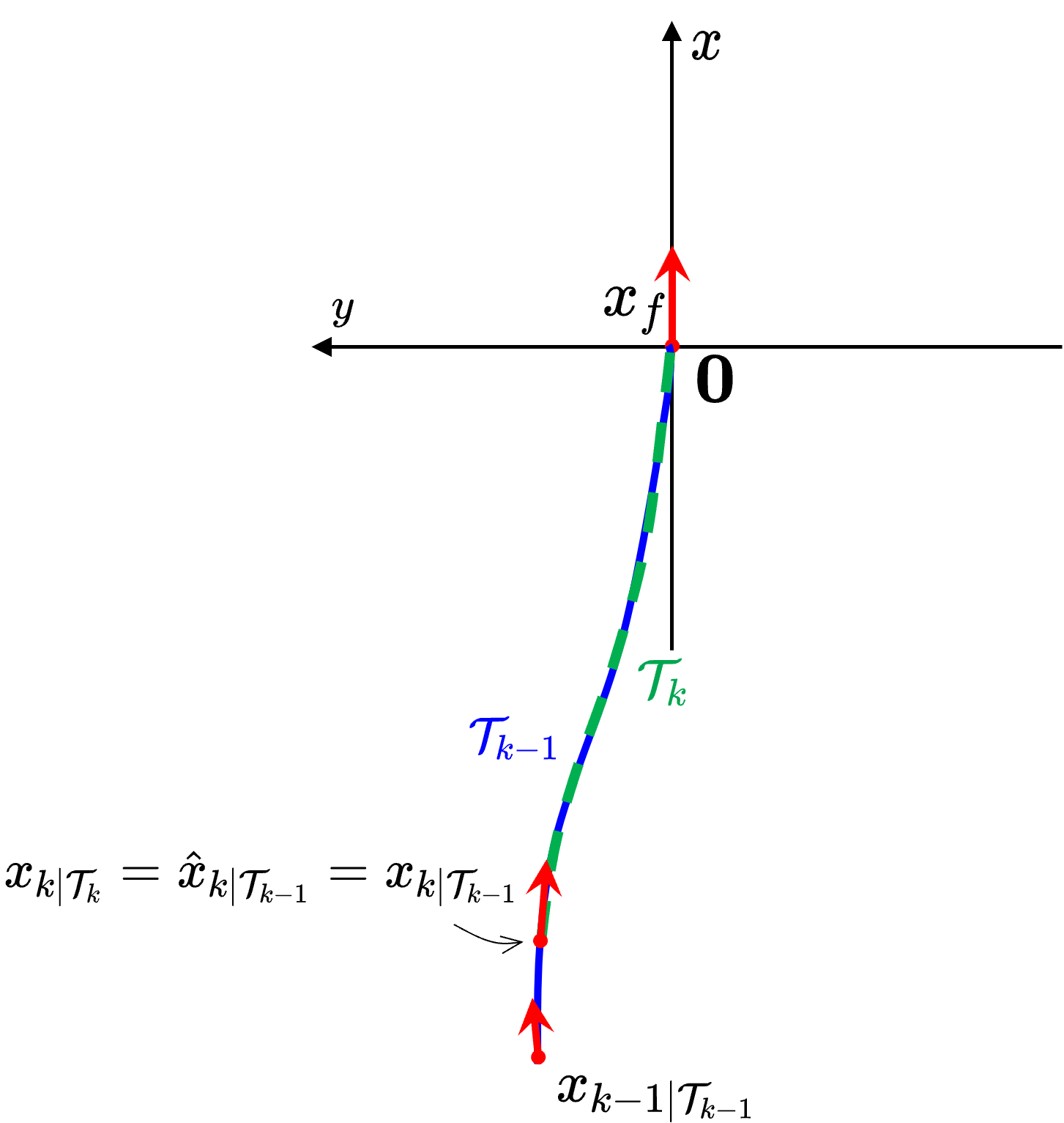}%
\label{fig:err_a}}
\hfil
\subfloat[Estimation error $\varepsilon \ne 0$]{\includegraphics[height=3in,keepaspectratio]{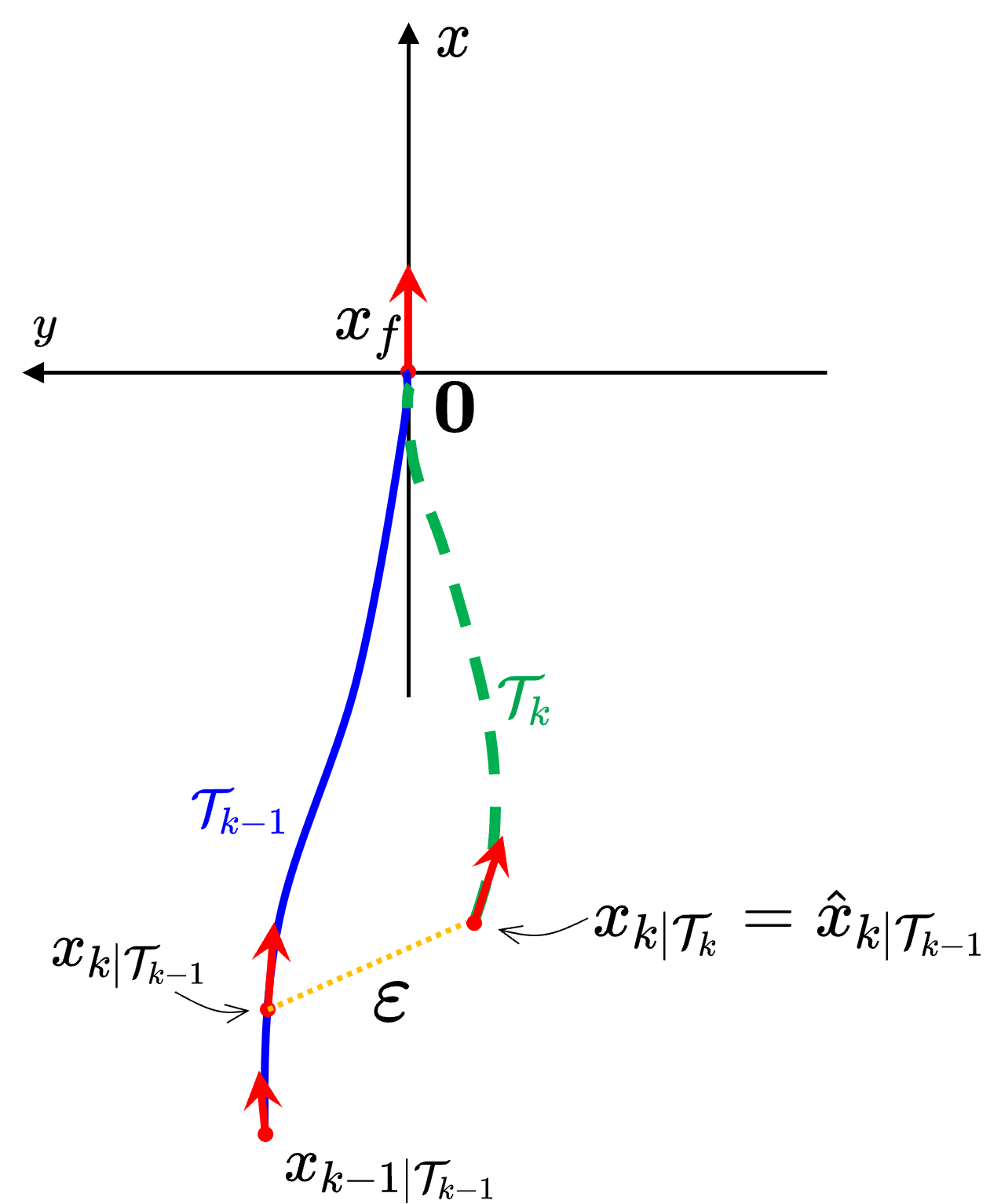}%
\label{fig:err_b}}
\caption{Illustration of effects of estimation error of state change $\varepsilon$ on continuous planning results at time step $t_{k-1}$ and $t_k$.}
\label{fig:err}
\end{figure*}

Figure \ref{fig:err_a} shows that without estimation error of the state change, trajectory $\mathcal T_{k}$'s start point $x_{k|\mathcal T_k}$ has the same state as the planned next-time-step state $x_{k|\mathcal T_{k-1}}$ by last trajectory $\mathcal T_{k-1}$, which is the prerequisite to realize the consistency of planning between frames. This can be proved as follows. From Equation \eqref{eq:coord}, we have: 

\begin{align}
^{P_k} x_{k|\mathcal T_{k-1}} = ^{P_{k-1}} x_{k|\mathcal T_{k-1}} - ^{P_{k-1}} P_k, \label{eq:xk_coord_1} \\
^{P_k} \hat x_{k|\mathcal T_{k-1}} = ^{P_{k-1}} x_{k|\mathcal T_{k-1}} - \Delta \hat P_k, \label{eq:xk_coord_2}
\end{align}
The estimation error of the state change between neighbor frames thus comes from:

\begin{equation}
\label{eq:epsilon_err}
\varepsilon = ^{P_{k-1}} P_k - \Delta \hat P_k. 
\end{equation}

In the case of Figure \ref{fig:err_a} when estimation error $\varepsilon=0$, $\Delta \hat P_k = ^{P_{k-1}} P_k$. In this case $\hat x_{k|\mathcal T_{k-1}}$ used as the initial planning start state $x_{k|\mathcal{T}_k}$ at time $t_k$ has the relationship: 

\begin{equation}
x_{k|\mathcal T_k} = \hat x_{k|\mathcal T_{k-1}} = x_{k|\mathcal T_{k-1}}. 
\end{equation}

In the case $\varepsilon \ne 0$ in \eqref{eq:epsilon_err}, minus Equation \eqref{eq:xk_coord_1} from \eqref{eq:xk_coord_2} we have: 

\begin{align}
^{P_k} \hat x_{k|\mathcal T_{k-1}} &=  ^{P_k} x_{k|\mathcal T_{k-1}} + (^{P_{k-1}} P_k -  \Delta \hat P_k), \\
&= ^{P_k} x_{k|\mathcal T_{k-1}} + \varepsilon, 
\end{align}
and therefore 
\begin{equation}
\hat x_{k|\mathcal T_{k-1}} = x_{k|\mathcal T_{k-1}} + \varepsilon,  
\label{eq:x_estimate}
\end{equation}
under the global coordinate system. 

Under this case in Figure \ref{fig:err_b}, the start state $x_{k|\mathcal T_k} $ for $\mathcal T_k$ deviates from the planned next-time-step state $x_{k|\mathcal T_{k-1}}$ at $\mathcal T_{k-1}$. The deviation between $x_{k|\mathcal T_k} $ and $x_{k|\mathcal T_{k-1}}$ is exactly the estimation error $\varepsilon$ of state change. 

Also note that both trajectories leads to the same terminal state $x_f$ despite existence of the error term $\varepsilon$. This is because that the terminal state is assumed fixed for stability analysis purpose as we mentioned before. Although different local vehicle coordinate system are used to represent the terminal state at each frame during planning, i.e. $^{P_{k-1}} x_f$ vs $^{P_k} x_f$, the coordinate transformation itself does not change any global object's state including $x_f$. Hence $x_f$ is not affected by the estimation error $\varepsilon$ of state change and every trajectory at each time step attempts to reach this destination $x_f$. 

It's likely that under the case Figure \ref{fig:err_b}, if the error term $\varepsilon$ is large enough, the continuously planned trajectory may never converge to the terminal state $x_f$ as $\varepsilon$ may drag the planning start point further and further away from $x_f$. Such inference is intuitive but lacks the support of theoretical derivation. We are interested in the questions: Is terminal state reachable during continuous planning given the estimation error $\varepsilon$ of state change? What is the bounds allowed for $\varepsilon$?

\subsection{Stability Analysis}

In the presence of estimation error $\varepsilon$, we have proved in last subsection that: 
\begin{align}
    x_{k|\mathcal T_k} &= \hat x_{k|\mathcal T_{k-1}} = x_{k|\mathcal T_{k-1}} + \varepsilon, \\ 
    &= x_{k-1 |\mathcal T_{k-1}} +\underbrace{ \Delta x_{k|\mathcal T_{k-1}} }_{\substack{\text{Planned state change}\\ \text{at trajectory }\mathcal T_{k-1}}}  + \underbrace{\varepsilon}_{\substack{\text{Estimation error of}\\ \text{the state change}}},   \label{eq:system_pre}
\end{align}
as illustrated in Figure \ref{fig:err_b}. 

Denote $x_{k|\mathcal T_{k}}$ by $x_k$ in the above equation, then we have the following discrete-time system described by:  
\begin{align}
    x_k = f(x_{k-1}) &= x_{k-1} + \underbrace{ \Delta x_{k|\mathcal T_{k-1}} + \varepsilon }  \\
                     &= x_{k-1} + \phantom{ \Delta x_{k|} }\hspace{0.5ex} \rho_{k-1},  \label{eq:system}
\end{align}
where $f:D\rightarrow \mathbb{R}$ is locally Lipschitz in $D\subset \mathbb{R}^3$, $D$ an open set containing the origin $\mathbf{0}  \in D$. 

The feasibility problem of the continuously trajectory planning then become the stability analysis for the discrete-time orbit, i.e. the sequence of state $x_k$ starting from an initial state $x_0$. 

Suppose  $f$ has an equilibrium at $x_f=\mathbf 0$, then the equilibrium $\mathbf{0}$ is said to be locally \textbf{Lyapunov stable} if: 

\begin{list}{}{}
\item{ For every $r>0$, there exists a $\delta > 0$ such that, if $\left \| x_0 -  \mathbf{0} \right \| < \delta $, then $\left \| x_k - \mathbf 0 \right \| <r$ for every $k>=0$. }  
\end{list}

Figure \ref{fig:ly_stable} shows an exemplary sequence of discrete state $x(\cdot )$ confined in the open ball of radius $r$, $B_r= \{x \in \mathbb{R}^3 \mid \left \| x \right \| < r \}$, projected in 2D plane. 

\begin{figure}
    \centering
    \includegraphics[height=0.6\linewidth]{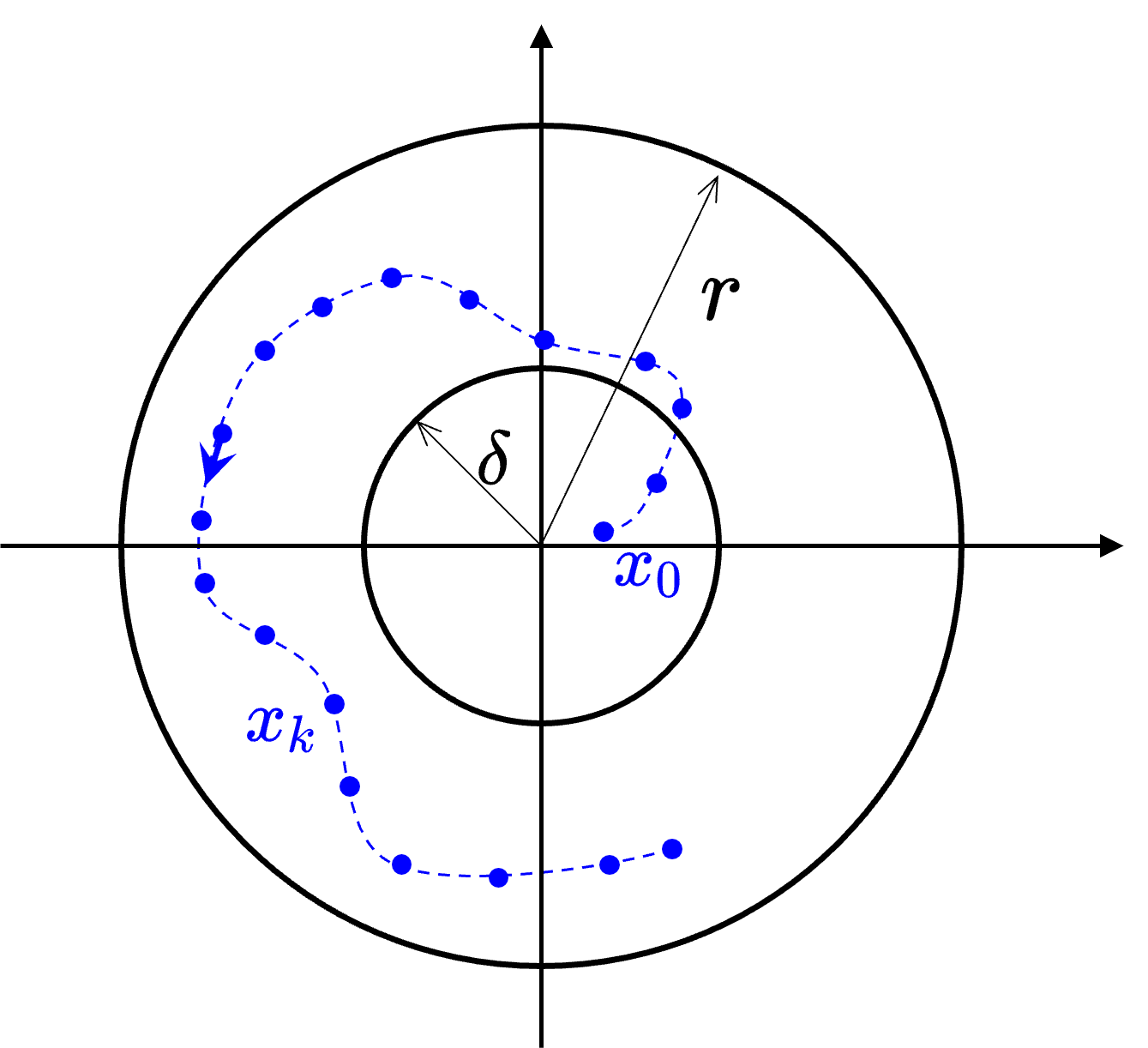}
    \caption{Stability in the sense of Lyapunov for discrete-time system projected in 2D plane.}
    \label{fig:ly_stable}
\end{figure}


Define a Lyapunov-alike function $V:D\rightarrow \mathbb{R}$ locally Lipschitz in $D$ with the form: 

\begin{equation}
    V(x) = x^Tx, x \in D,
\label{eq:func_V1}
\end{equation}
which satisfies the properties: 
\begin{equation}
    V(\mathbf 0) = 0, \text{ and } V(x) > 0, \forall x \in D-{\mathbf 0}. 
\label{eq:func_V2}
\end{equation}

Given the $(k-1)$-th state $x_{k-1}$, the value change of function  $V:D\rightarrow \mathbb{R}$  from $x_{k-1}$ to $x_k$ is: 
\begin{align}
    \Delta V(x_{k-1}) &= V(f(x_{k-1})) - V(x_{k-1}) \\
                      &= \left ( x_{k-1} + \rho_{k-1} \right ) ^T \left ( x_{k-1} + \rho_{k-1} \right ) - x_{k-1}^Tx_{k-1} \\
                      &= \left ( 2x_{k-1} + \rho_{k-1} \right ) ^T \rho_{k-1}.  \label{eq:delta_v_3}
\end{align}

We assume for $\Delta V(x_{k-1})$, the following prerequisite is satisfied: 

\begin{prerequisite}
    $\exists \eta >0$, such that $\forall x_{k-1} \in  \{x \in D \mid \left \| x \right \| > \eta  \}$, $\Delta V(x_{k-1}) \le 0 $ is always satisfied, given the Lipschitz-continuous function $V:D\rightarrow \mathbb{R}$ defined in Equation \eqref{eq:func_V1} and \eqref{eq:func_V2}. 
\label{prereq:delta_V}
\end{prerequisite}

\begin{remark}
    Prerequisite \ref{prereq:delta_V} is not stringent in the context of the continuously trajectory planning problem. We will show in the following that under certain conditions easy to met, the prerequisite assumed can be guaranteed. 
\end{remark}


To satisfy $\Delta V(x) \le 0$, from Equation \eqref{eq:delta_v_3} the inner product of $(2x+\rho)$ and $\rho$ have to be no greater than 0. Figure \ref{fig:inner_product} shows the physical meaning of this in the Euclidean plane. It shows different possibilities of $\rho$ and how it affects the $2x+\rho$ and correspondingly their inner product. 

\begin{figure}
    \centering
    \includegraphics[width=0.7\linewidth]{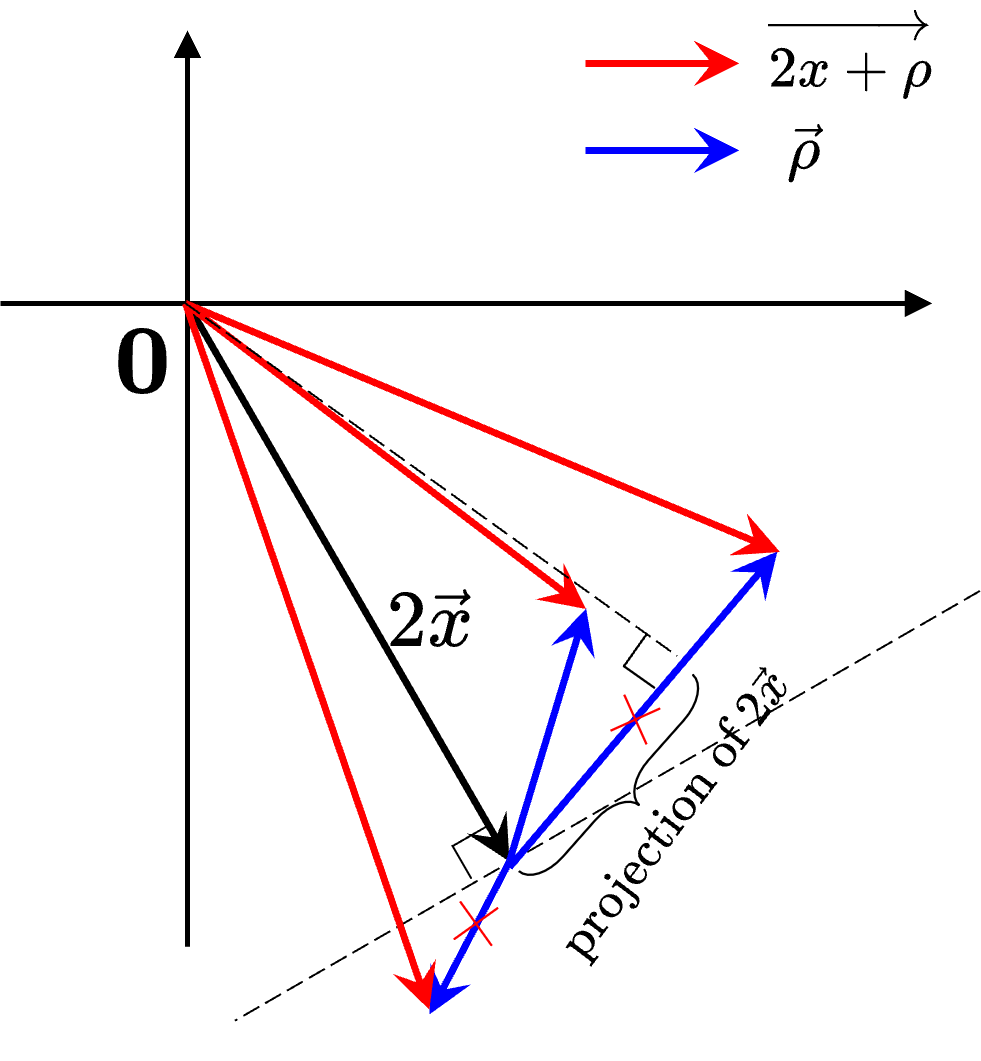}
    \caption{Inner product of $\rho$ and $2x+\rho$ at different combinations in Euclidean plane.}
    \label{fig:inner_product}
\end{figure}

Of the three examples given in the figure, either $\rho$ too lengthy or its wrong direction leads to $\overrightarrow{2x+\rho} \cdot \overrightarrow{\rho} > 0$. It is straightforward that two conditions have to be met to satisfy non-positive inner product in Euclidean plane: (1) $\overrightarrow{\rho} \cdot \overrightarrow x \le 0$; and (2) $\left | \overrightarrow{\rho} \right |$ should be less than the projection of $2\overrightarrow{x}$ on $\overrightarrow{\rho}$. For the vector space $\mathbb{R}^3$, similarly we concluded that the following conditions have to be satisfied to ensure $\Delta V(x) \le 0$: 

\begin{equation}
    x^T\rho \le 0,  \text{ and }  \rho^T \rho < -2x^T \rho. 
\label{eq:rho}
\end{equation}

We then show that condition \eqref{eq:rho} can be met in the following discussion. The term $\rho$ contains the estimation error of state change, $\varepsilon$, as seen from Equation \eqref{eq:system_pre} to \eqref{eq:system}. The error $\varepsilon$ is in probabilistic distribution due to noise and error from the motion sensors. Sensor noise and error is usually small and hence it is safe to assume a upper bound that:  
\begin{list}{}{}
\item{
    $\exists  \bar \varepsilon$, such that $\varepsilon$ is bounded $\left \| \varepsilon \right \| < \bar \varepsilon$. 
}
\end{list}

The other term contained in $\rho$ is $\Delta x_{k|\mathcal T_{k-1}}$, which is the planned next-step state change at $\mathcal T_{k-1}$. Figure \ref{fig:epsilon_bound} shows a possible vector of $\Delta x_{k|\mathcal T_{k-1}}$ given the $(k-1)$-th step's planning initial point $x_{k-1}$ under Euclidean plane as an example. Due to the nature of trajectory planning, the planned trajectory steps towards the terminal state $\mathbf 0$. Therefore, we can assume the following relation: 

\begin{equation}
    \Delta x_{k|\mathcal T_{k-1}} ^T x_{k-1} < 0.  
\label{eq:dir_assumption}
\end{equation}

\begin{figure}
    \centering
    \includegraphics[width=0.7\linewidth]{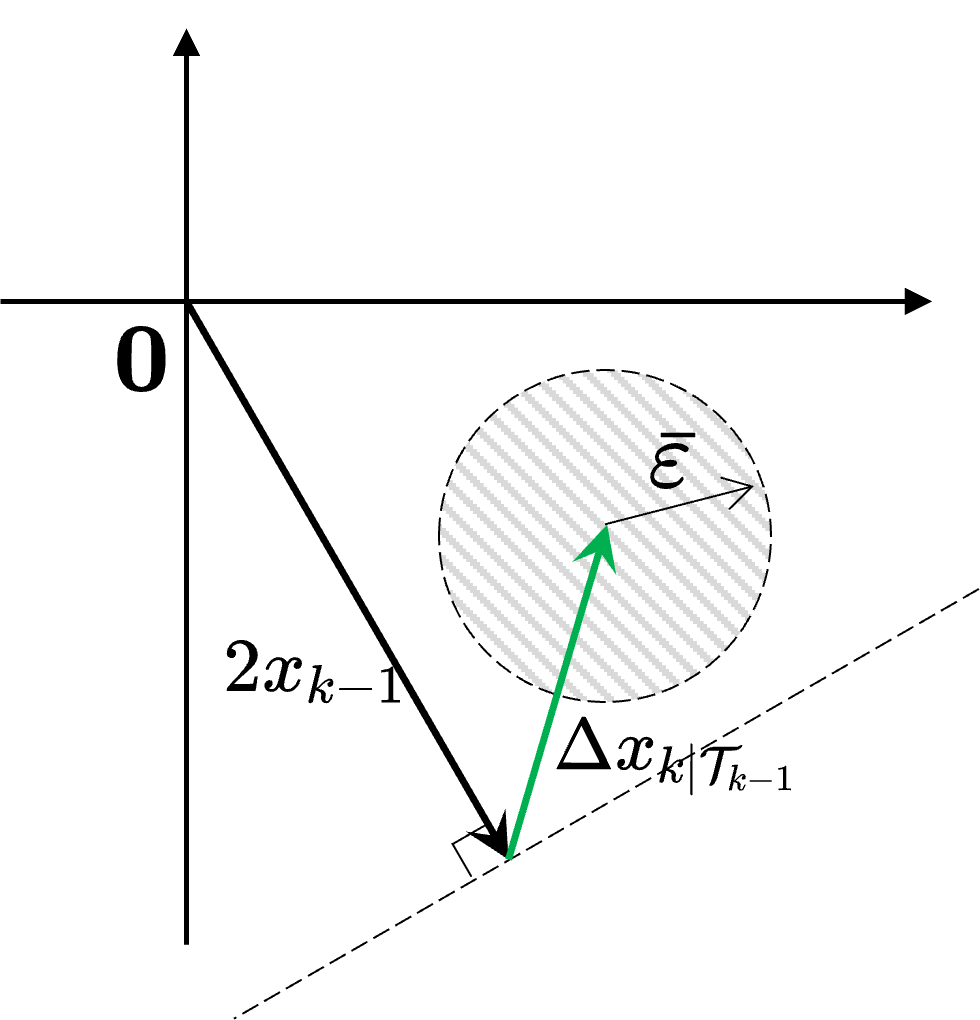}
    \caption{Example of $\Delta x_{k|\mathcal T_{k-1}}$ and bounds of $\varepsilon$ in Euclidean plane.}
    \label{fig:epsilon_bound}
\end{figure}

In Figure \ref{fig:epsilon_bound}, the open ball of radius $\bar \varepsilon$ is shown in the grey shaded area atop the tip of $\Delta x_{k|\mathcal T_{k-1}}$. Then the vector $\rho$, i.e. the sum of $\Delta x_{k|\mathcal T_{k-1}}$ and $\varepsilon$, is known to be confined in this shaded area. If for any $\varepsilon$ that bounded in $\left \| \varepsilon \right \| < \bar \varepsilon$, the condition \eqref{eq:rho} could be satisfied, then $\Delta V(x) \le 0$ can be guaranteed. 

This is very likely to be satisfied when the error bound $\bar \varepsilon$ is much smaller compared to the norm $ \left \| x_{k-1} \right \|$ or $\left \|  \Delta x_{k|\mathcal T_{k-1}}  \right \|$ if the relation \eqref{eq:dir_assumption} is met. Therefore, we can make the assumption that $\exists \eta > 0$, for any $x_{k-1} \in D$ that satisfies $\left \| x_{k-1} \right \| > \eta$, condition \eqref{eq:rho} is always met, and hence $\Delta V(x_{k-1}) \le 0$. The above discussion shows how the Prerequisite \ref{prereq:delta_V} is made.  

\begin{prerequisite}
    In the discrete system \eqref{eq:system}, $\rho$ is bounded within an open ball with radius $\bar \rho$, i.e. $\left \| \rho \right \| < \bar \rho$. 
\label{prereq:rho_bound}
\end{prerequisite}

\begin{proof}
    During planning, the vehicle's own physical motion capabilities would be considered to limit the planned state change $\Delta x_{k | \mathcal T_{k-1}} $, i.e. 
    \begin{equation}
        \left \| \Delta x_{k | \mathcal T_{k-1}} \right \| <  \Delta \bar x. 
    \end{equation}    
    And hence,  
    \begin{equation}
        \left \| \rho \right \| \le \left \| \Delta x_{k | \mathcal T_{k-1}} \right \|  + \left \| \varepsilon \right \| < \Delta \bar x + \bar \varepsilon    
    \end{equation}
    
    Therefore, there must exist a $\bar \rho$, such that $\left \| \rho \right \| < \bar \rho$ and $\bar \rho \le \Delta \bar x + \bar \varepsilon$. 
\end{proof}

\begin{remark}
    For the trajectory planning problem, the terminal state $x_f = \mathbf{0}$ may not be an equilibrium point since $f(\mathbf 0) = \mathbf 0$ is not guaranteed due to the term $\rho$ in Equation \eqref{eq:system}. However, from Prerequisite \ref{prereq:rho_bound}, we know that next state from $\mathbf 0$ is nearby, i.e. $\left \| f(\mathbf 0) \right \| \le \bar \rho$. 
\end{remark}

\begin{claim}
    For the discrete system described in Equation \eqref{eq:system}, with Prerequisite \ref{prereq:delta_V} and \ref{prereq:rho_bound}, the Lyapunov stability could not be achieved; but a weaker conclusion could be made: There exists a $\delta > 0$, if $\left \| x_0 -  \mathbf{0} \right \| < \delta $, then $ x_k $ is bounded in the sense that $\left \| x_k - \mathbf 0 \right \| < r$ for every $k \ge 0$ for some $r$. 
\label{clm:bounded}
\end{claim}

\begin{proof}
Choose $s = max \left \{ \eta, \bar \rho \right \}$ where $\eta$ and $\bar \rho$ are declared in Prerequisite \ref{prereq:delta_V} and \ref{prereq:rho_bound}, such that the open ball $B_s = \left \{  x \in \mathbb{R}^3 \mid \left \| x \right \| < s \right \} \subset D$. Then choose $r > s + \bar \rho> 0$ that $B_r = \left \{  x \in \mathbb{R}^3 \mid \left \| x \right \| < r \right \} \subset D$. Let $\alpha = \min_{\left \| x \right \| = r} V(x)$, then we know $\alpha > 0$ due to \eqref{eq:func_V2}. Take $\beta \in (0, \alpha)$, the set $\Omega_\beta = B_r \cap V^{-1}( [0,\beta]) \subset B_r$ could have several connected components as shaded in Figure \ref{fig:fig_set}. 

Consider $C_\beta \subset \Omega_\beta$ is the connected component that contains $B_s$, i.e. $B_s \subset C_\beta$. Since the function has the form as in \eqref{eq:func_V2}, this can be ensured with a chosen $\beta \ge (s + \bar \rho)^2$. In the following, we will prove that $f^n(C_\beta) \subset C_\beta$ for every $n \ge 0$. 

First we show that the next discrete state from $\mathbf 0$ still belongs to $C_\beta$, i.e. $f(\mathbf 0) \in C_\beta$. From Prerequisite \ref{prereq:rho_bound}, we have $\left \| f(\mathbf 0) \right \| < \bar \rho \le s$. Therefore, $f(\mathbf 0) \in B_s \subset C_\beta$. 

Next we prove that $f(C_\beta) \subset B_r \cap V^{-1}([0,\beta])$. This has to be discussed for $B_s$ and $C_\beta \backslash B_s$ separately. For $x \in B_s$, 
\begin{equation}
    \left \| f(x) \right \| \le \left \| x \right \| + \left \| \rho \right \| < s + \bar \rho < r,
\end{equation}
and hence, 
\begin{equation}
    f^T(x)f(x) < (s + \bar \rho)^2 \le \beta. 
\end{equation}
Therefore, $f(B_s) \subset B_r \cap V^{-1}([0,\beta])$ . 

Then since $f:D\rightarrow \mathbb{R}^3$ is Lipschitz in $D$, $f(C_\beta)$ is also connected and $f(B_s) \subset f(C_\beta)$. This means that at least part of $f(C_\beta \backslash B_s)$ is a subset of $B_r$ too. We can further conclude that $f(C_\beta \backslash B_s) \subset B_r$. If this is not true, then $f(C_\beta \backslash B_s)$ overlaps with $B_r$. There's a point $x \in C_\beta \backslash B_s$ such that $\left \| f(x) \right \| = r$. Then the Lyapunov-alike function 
\begin{equation}
    V(f(x)) \ge \alpha > \beta \ge V(x).     
\end{equation}
This is contradictory with the non-increasing characteristics of $V(x)$ in Prerequisite \ref{prereq:delta_V}. 

Thus $f(C_\beta)$ is connected and a subset in $B_r \cap V^{-1}([0,\beta])$. Meanwhile $f(\mathbf 0) \in f(C_\beta)$ and  $f(\mathbf 0) \in C_\beta$. This implies that  $f(C_\beta) \subset C_\beta$. We can then conclude that $f^n(C_\beta) \subset C_\beta$ for any $n \in \mathbb N$. Hence we can choose a $\delta > 0$ that satisfies $\left \{ x \in D \mid \left \| x \right \| < \delta \right \} \subset C_\beta$, if $\left \| x \right \| < \delta$, then $\left \| f^n(x) \right \| < r$. 

\end{proof}

\begin{figure}
\centering
\includegraphics[height=0.6\linewidth]{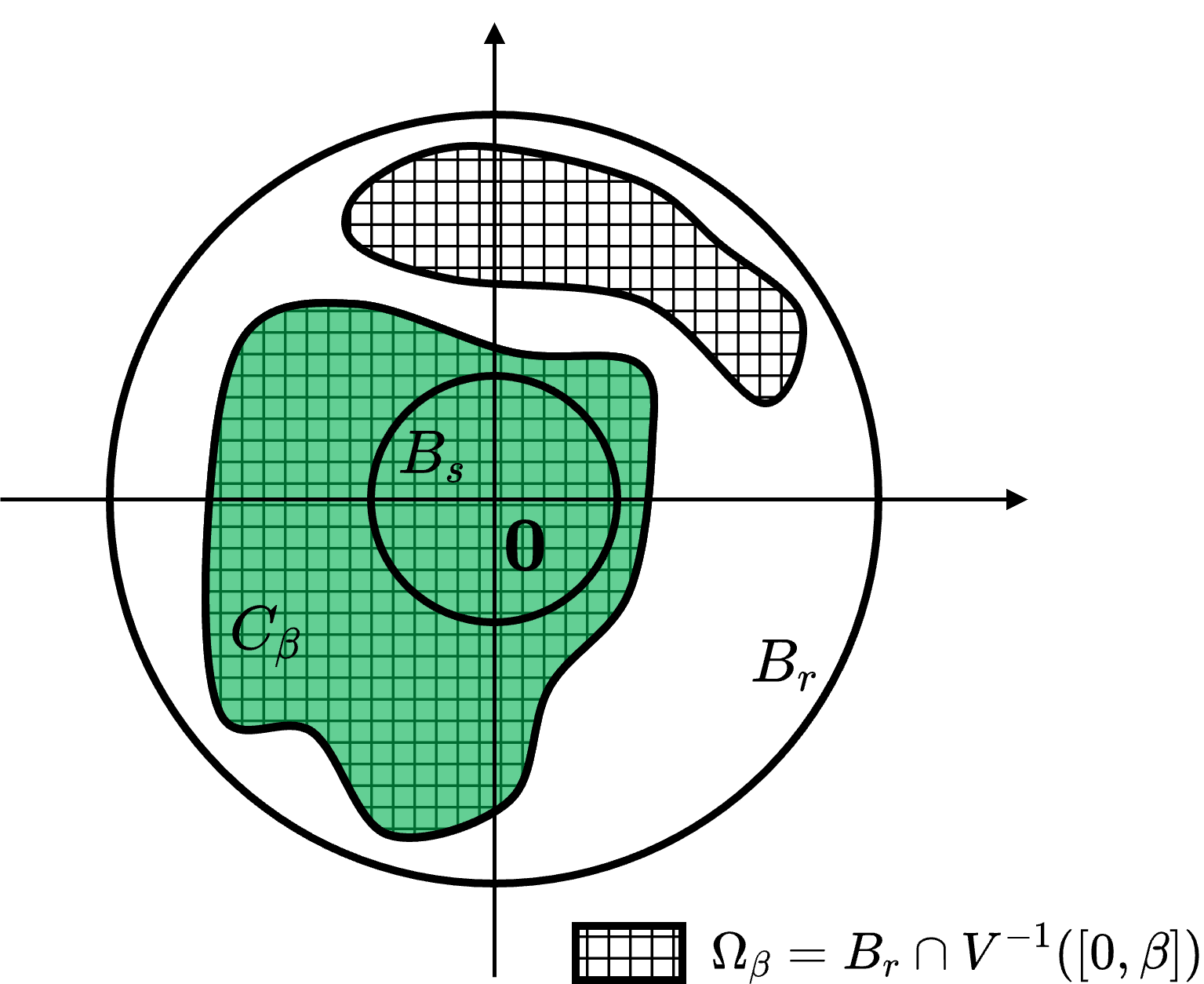}
\caption{Illustration of the sets projection to 2D plane}
\label{fig:fig_set}
\end{figure}


\begin{claim}
    If $\Delta V(x)$ is strictly decreasing in Prerequisite \ref{prereq:delta_V}, for the system \eqref{eq:system} along with Prerequisite \ref{prereq:rho_bound}, there exists $\delta > 0$, such that if $\left \| x_0 -  \mathbf{0} \right \| < \delta $, then $\lim_{n \to \infty}{ \left \| x_n -  \mathbf{0} \right \| } < \eta + \bar \rho$. This means that the state $x_k$ will be contained in $B_{\eta+\bar\rho} = \{x \in D \mid \left \| x \right \| < \eta + \bar \rho \}$ at some point, and its orbit will remain inside thereafter. 
\label{clm:further_bounded}
\end{claim}

\begin{proof}
    From the proof part for Claim \ref{clm:bounded}, choose $\delta$ that satisfies $\left \{ x \in D \mid \left \| x \right \| < \delta \right \} \subset C_\beta$, if  $\left \| x \right \| < \delta$, then $ f^n(x) \in C_\beta$ for every $n \ge 0$. 
    
    First we show that for $x \in C_\beta \backslash B_\eta$, at some point k, $f^k(x) \in B_\eta$, where $B_\eta = \{ x \in D \mid \left \| x \right \| < \eta \}$. For the sake of contradiction, suppose that this is not the case, then for all $k \ge 0$ we have 
    \begin{equation}
        f^k(x) \in C_\beta \backslash B_\eta. 
    \end{equation}
    
    Since $ C_\beta \backslash B_\eta $ is compact, and $\Delta V$ is continuous and $\Delta V(x) < 0$ for $ x \in C_\beta \backslash B_\eta $, then from Weierstrass theorem, we know $\Delta V$ attains a negative maximum $-\mu$, i.e.
    \begin{equation}
        \Delta V (x) \le -\mu < 0, \text{ if } x \in C_\beta \backslash B_\eta, 
    \end{equation}
    and hence, 
    \begin{align}
        V(f^n(x)) &= V(f^{n-1}(x)) + \Delta V(f^{n-1}(x))   \\
                &= V(f^{n-2}(x)) + \Delta V(f^{n-2}(x)) + \Delta V(f^{n-1}(x)) \\
                & \dots   \\
                &= V(x)  + \sum_{k=1}^{n-1} \Delta V\left ( {f^k(x)}  \right ) \\
                &\le V(x) - (n-1)\mu. 
    \end{align}
    
    Letting $n\rightarrow +\infty $, the right-hand side tends to $ -\infty $. This is in contradiction with $V(x) > 0$ for $x \in D - \mathbf 0$. Therefore, for $x \in C_\beta \backslash B_\eta$, there must exist some $k$ such that,
    \begin{equation}
        f^k(x) \in B_\eta. 
    \label{eq:tmp}
    \end{equation}

    For $x \in B_\eta$, 
    \begin{align}
        \left \| f(x) \right \| &\le \left \| x \right \| + \left \| \rho \right \|  \\
                                &<  \left \| x \right \| + \bar \rho. 
    \end{align}
    This shows that $f(x) \in B_{\eta + \bar\rho}$ if $x \in B_\eta$. Together with the summary made in \eqref{eq:tmp}, the Claim \ref{clm:further_bounded} can be proved. 
    
\end{proof}

\begin{remark}
   Claim \ref{clm:further_bounded} further extends the boundedness conclusion made in Claim \ref{clm:bounded}, and shows that for the trajectory planning problem, as $k \rightarrow +\infty$, the planned state will eventually contained in $B_{\eta+\bar\rho}$, despite the state change estimate error. The $B_{\eta+\bar\rho}$ is a relatively small area and should satisfy the need for the trajectory planning problem. 
\end{remark}

\section{Results and Discussions}

Two scenarios are designed to experiment the local planning method without global positioning information. Simulations of two scenes were conducted with the validation pipeline developed in subsection \ref{subsec_valid_pip} \footnote{Code for this work is available at \url{https://github.com/codezs09/l2_frenet_planner.git}}. 

To experiment the effects of sensor drift and noise to the planning results, the offset and standard deviation of speed and yaw rate sensor readings in \eqref{eq:speed_error} and \eqref{eq:yaw_rate_error} are set to:
\begin{align}
v_\text{offset} &= -0.1 \text{ m/s}, \quad \sigma_v = 0.1 \text{ m/s}. \label{eq:v_err_val} \\
\dot{\theta}_\text{offset} &= 0.57 \text{ deg/s}, \quad \sigma_{\dot{\theta}} = 1.72 \text{ deg/s}.  \label{eq:theta_err_val}
\end{align}

These errors are set to large values intentionally to validate the feasibility of the proposed local planning methodology. Normally the speed and yaw rate sensor would achieve more accurate readings with respect to the ground truth. 

\subsection{Moving Traffic Scene}
In this scene, the ego car is set to run in a double-lane road with other moving vehicles. Figure \ref{fig:traffic_scene} shows the the continuous local planning results under the measurement error settings \eqref{eq:v_err_val} and \eqref{eq:theta_err_val}. The moving vehicles are shown in blue bounding boxes, which are enlarged to ensure safe distance with the ego car. The dashed-line boxes represent the predicted motions for these vehicles. The ego vehicle is represented in a green box in this figure. The candidate trajectory planned for each lane (left lane in blue, right lane in orange) with a planning horizon of 5.0 seconds is also displayed extended from the tail of the ego vehicle. The selected trajectory is highlighted in red points with each point representing an increment of 0.1 second per time step. This helps visualize the speed change by observing the density and spread of the selected trajectory. 


\begin{figure}
    \centering
    \subfloat{%
        \includegraphics[trim=1.2cm 0.8cm 1cm 0.5cm, clip, width=0.5\linewidth]{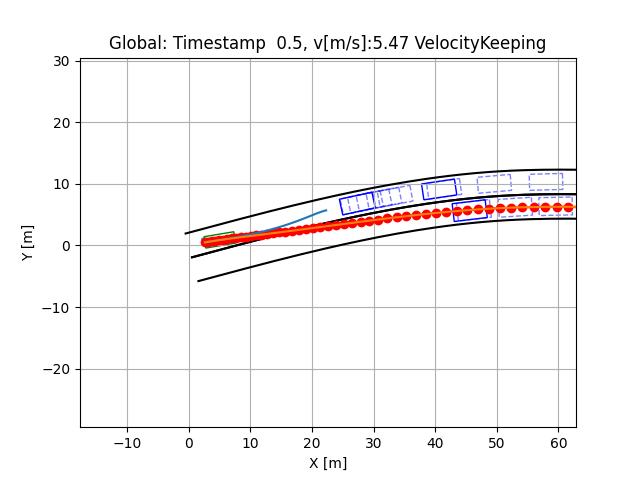}
    }
    \subfloat{%
        \includegraphics[trim=1.2cm 0.8cm 1cm 0.5cm, clip, width=0.5\linewidth]{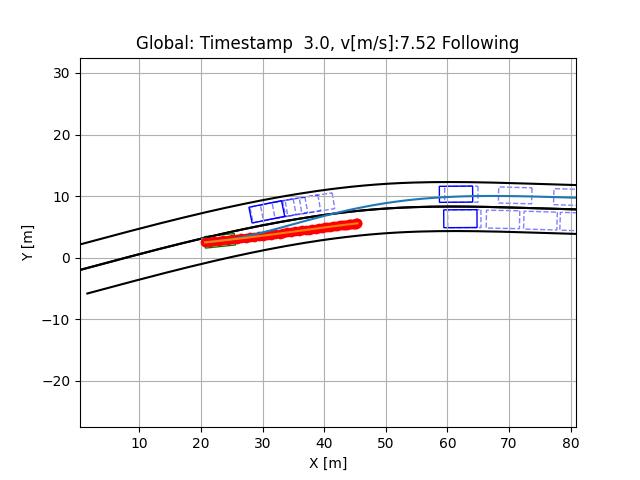}
    }
    \\
    \vspace{-0.3cm} 
    \subfloat{%
        \includegraphics[trim=1.2cm 0.8cm 1cm 0.5cm, clip, width=0.5\linewidth]{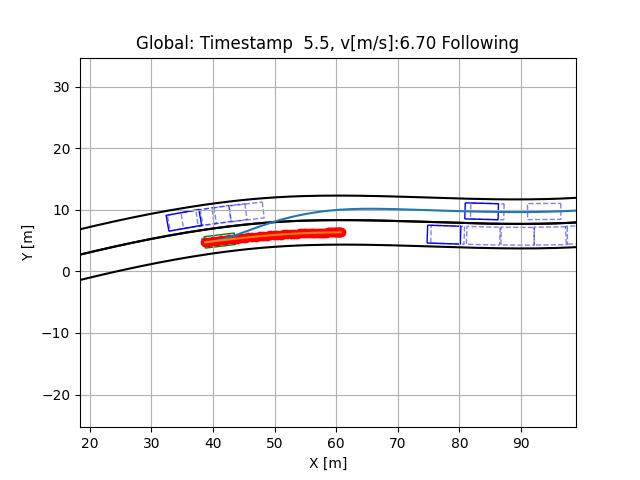}
    }
    \subfloat{%
        \includegraphics[trim=1.2cm 0.8cm 1cm 0.5cm, clip, width=0.5\linewidth]{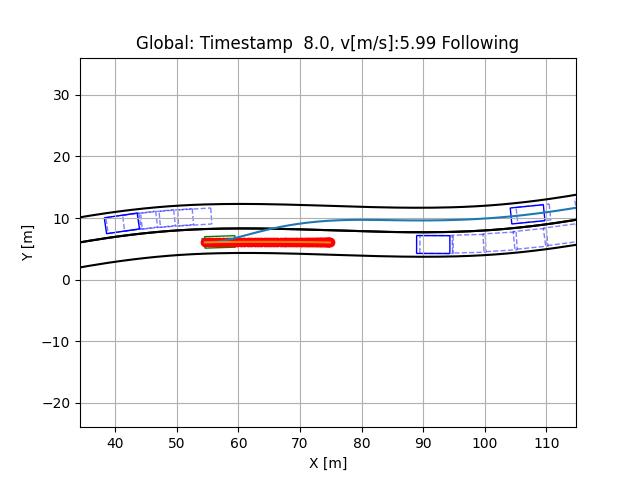}
    }
    \\
    \vspace{-0.3cm} 
    \subfloat{%
        \includegraphics[trim=1.2cm 0.8cm 1cm 0.5cm, clip, width=0.5\linewidth]{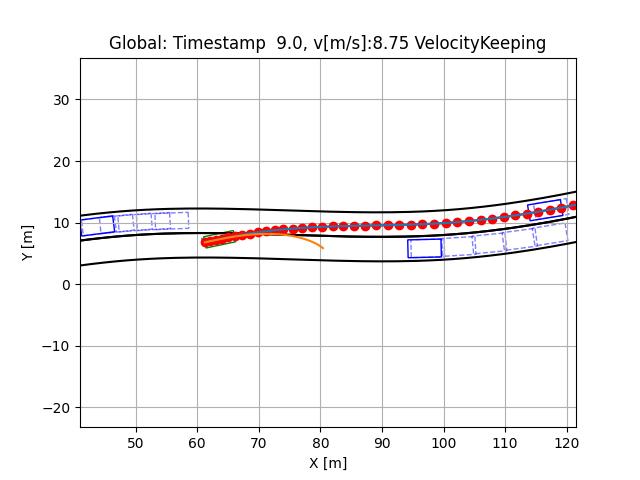}
    }
    \subfloat{%
        \includegraphics[trim=1.2cm 0.8cm 1cm 0.5cm, clip, width=0.5\linewidth]{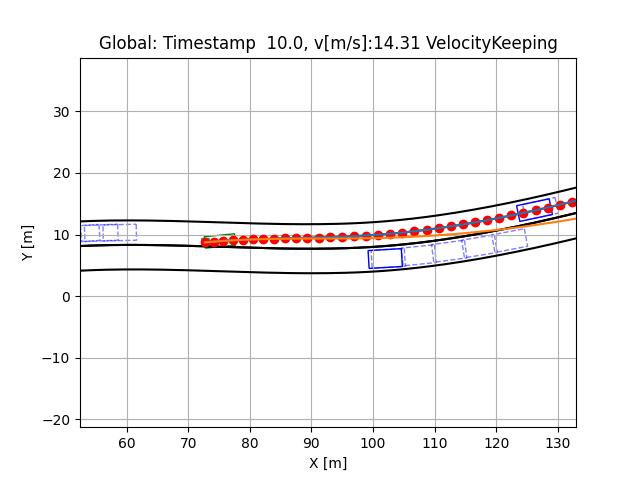}
    }
    \\
    \vspace{-0.3cm}
    \subfloat{%
        \includegraphics[trim=1.2cm 0.8cm 1cm 0.5cm, clip, width=0.5\linewidth]{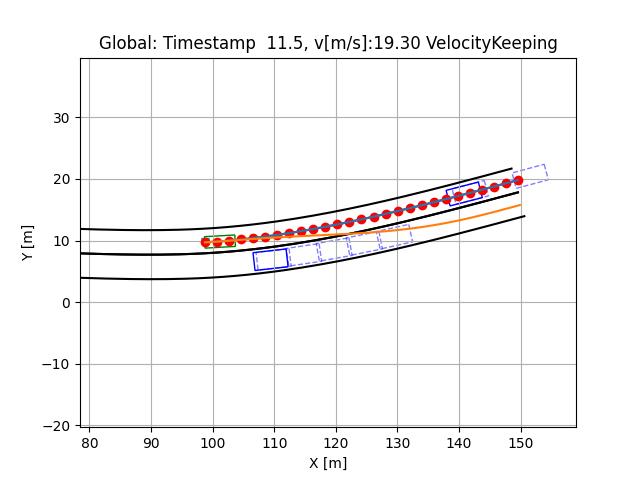}
    }
    \subfloat{%
        \includegraphics[trim=1.2cm 0.8cm 1cm 0.5cm, clip, width=0.5\linewidth]{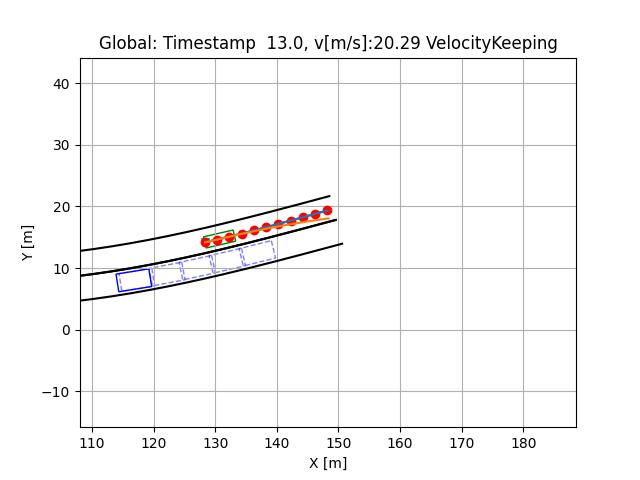}
    }
    \caption[Caption for LOF]{Local planning in a moving traffic scene\protect\footnotemark.}
    \label{fig:traffic_scene} 
\end{figure}
\footnotetext{The animation gif is available at \url{https://bit.ly/3L4QSHB}.}

In Figure \ref{fig:traffic_scene}, the ego car performs a right lane change first, and then follows the front car until the gap on the left lane is safe enough for it to make another lane change back to the left lane. The safety distance in both longitudinal and lateral directions of the ego car are well maintained. The ego car centers well within the lane bounds when not performing lane changing. This demonstrates that despite the unavailability of global localization information and a not-so-accurate sensor readings of speed and yaw rate to estimate the relative motion, the continuous planning remains highly feasible under the proposed methodology in Figure \ref{fig:fig_plan}.

Figure \ref{fig:traffic_sensor_reading} shows the change of speed and yaw rate of the ego car during the moving traffic scene, as well as the measurement readings in blue-cross line. Time plots of the measurement errors are also displayed in the right subplots to show the deviation and noise of the sensor readings. 

\begin{figure}
\centering
\includegraphics[trim=2.0cm 0cm 3.0cm 1cm, clip, width=\linewidth]{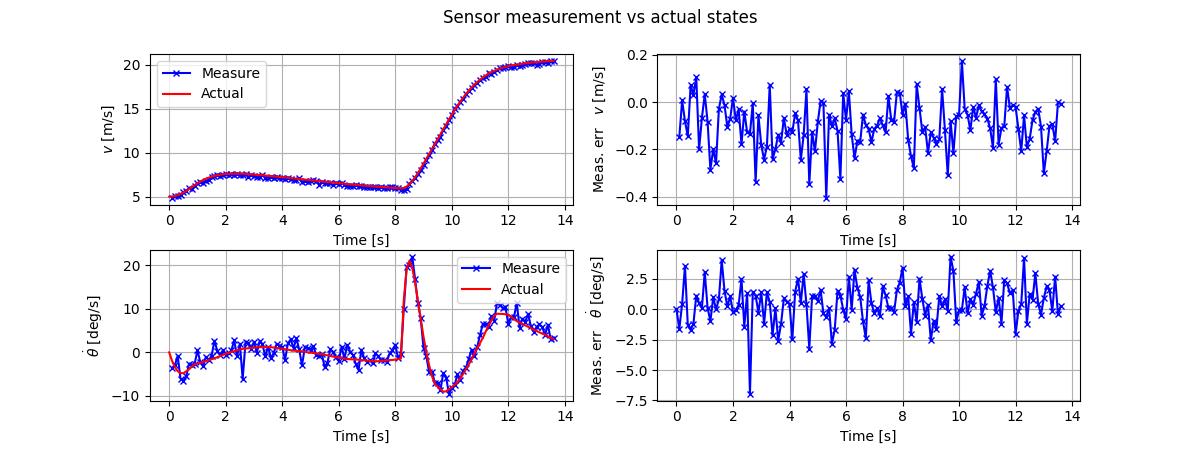}
\caption{Time plots of vehicle speed, yaw rate and measurement errors.}
\label{fig:traffic_sensor_reading}
\end{figure}

We also examined how the measurement errors from speed and yaw rate affect the relative motion estimate between planning frames. In Figure \ref{fig:traffic_err_box}, the top two plots shows the box plot as well as the scattering of the measurement errors. The bottom plots show how the above measurement errors resulted in the estimated errors $\varepsilon$ in Equation \eqref{eq:epsilon_err} for the pose change between planning frames, i.e. $\Delta x$, $\Delta y$, and $\Delta\theta$, under the local coordinate system.

\begin{figure}
\centering
\includegraphics[trim=0.0cm 0cm 1.5cm 1cm, clip, width=\linewidth]{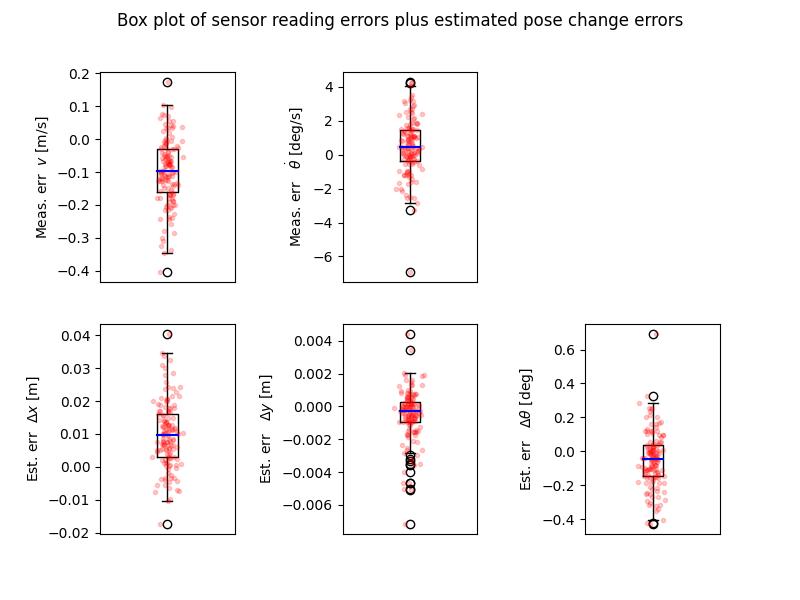}
\caption{Box plots of measurement errors (upper plots) and the resulted estimation errors of relative motion between planning frames (bottom plots)}
\label{fig:traffic_err_box}
\end{figure}

\subsection{Stop Scene}

Another scene where the vehicle has to stop to wait in traffic was ran with the same measurement errors settings for the speed and yaw rate. We are interested to see how the vehicle will perform under low speed or at static. In this scene, a fixed stop distance was set to observe whether the ego car could reach and maintain the target pose under the impact of measurement errors. 

\begin{figure}
    \centering
    \subfloat{%
        \includegraphics[trim=1.2cm 0.8cm 1cm 0.5cm, clip, width=0.5\linewidth]{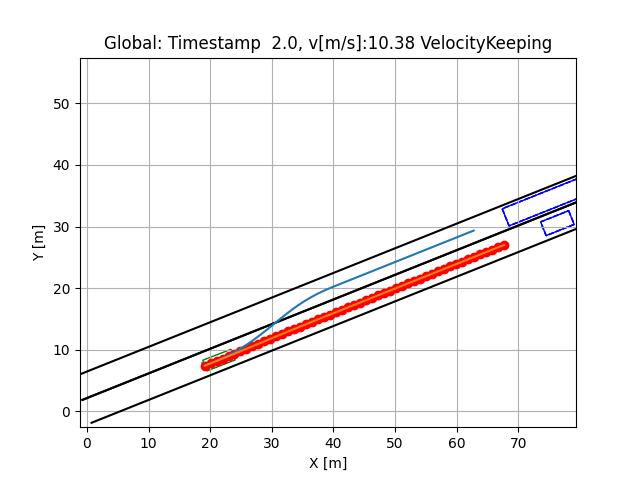}
    }
    \subfloat{%
        \includegraphics[trim=1.2cm 0.8cm 1cm 0.5cm, clip, width=0.5\linewidth]{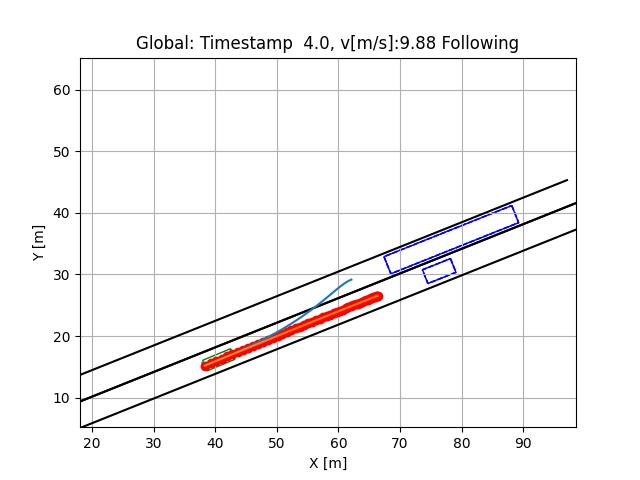}
    }
    \\
    \vspace{-0.3cm} 
    \subfloat{%
        \includegraphics[trim=1.2cm 0.8cm 1cm 0.5cm, clip, width=0.5\linewidth]{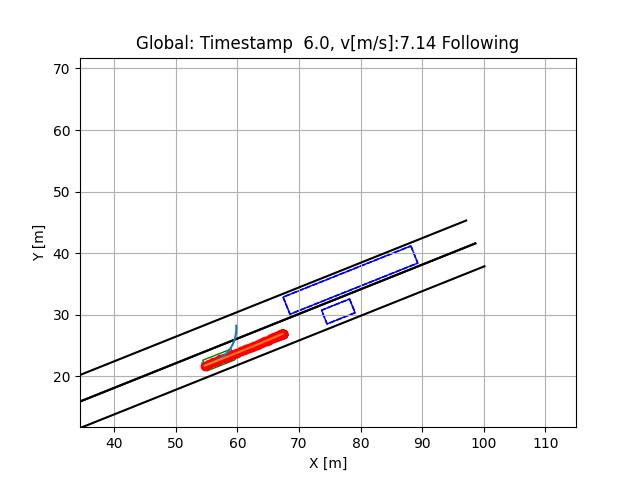}
    }
    \subfloat{%
        \includegraphics[trim=1.2cm 0.8cm 1cm 0.5cm, clip, width=0.5\linewidth]{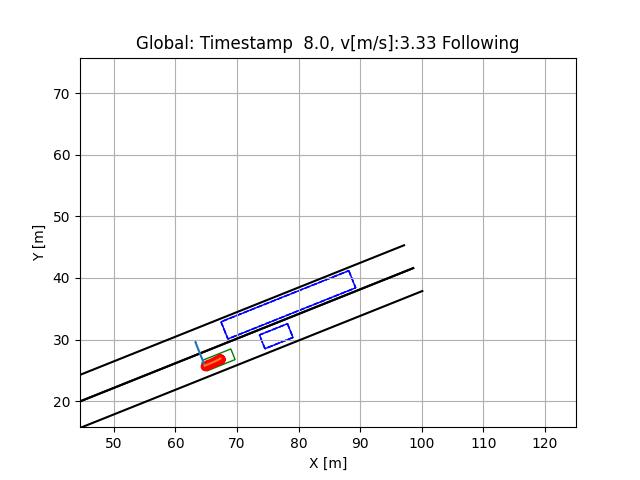}
    }
    \\
    \vspace{-0.3cm} 
    \subfloat{%
        \includegraphics[trim=1.2cm 0.8cm 1cm 0.5cm, clip, width=0.5\linewidth]{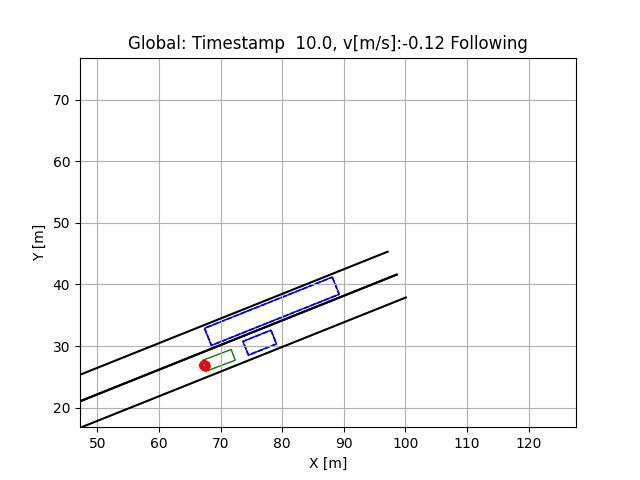}
    }
    \subfloat{%
        \includegraphics[trim=1.2cm 0.8cm 1cm 0.5cm, clip, width=0.5\linewidth]{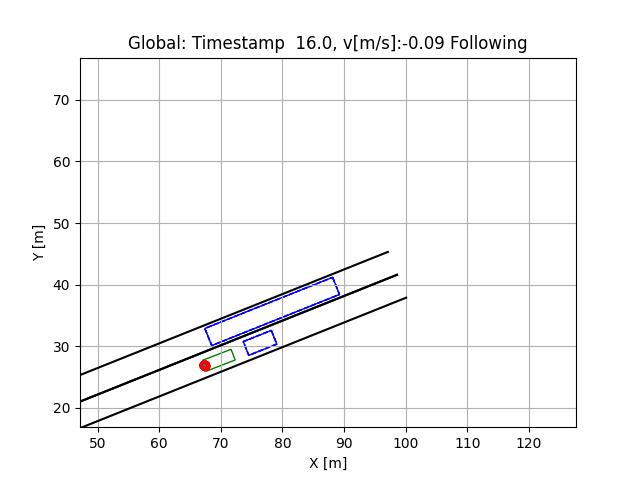}
    }
    \caption[]{Local planning in a stop scene\protect\footnotemark.}
    \label{fig:stop_scene} 
\end{figure}


Figure \ref{fig:stop_scene} shows that traffic in both lanes are blocked by a long semi-truck and a sedan in front. The ego car makes a deceleration to a full stop. \footnotetext{The animation gif is available at \url{https://bit.ly/47XMcxg}.} The bottom two plots demonstrate that the ego car is capable to maintain the stop pose despite the estimation error of relative motion change between frames caused by the sensor errors. Interesting readers can refer to the animation\protect\footnotemark to see how the planning compensates the estimation error and get a sense of how the condition \eqref{eq:rho} is met under this error setting. 

\footnotetext{The animation gif is available at \url{https://bit.ly/3sFb1hk}.}

Note that the speed planned is negative in both bottom plots in Figure \ref{fig:stop_scene} when the vehicle is stopped. This is due to that negative speed error offset $v_\text{offset} = -0.1 \text{ m/s}$ leads to a positive estimation error offset of $\Delta x$ along the longitudinal direction, as can be seen in Figure \ref{fig:traffic_err_box}. From Equation \eqref{eq:x_estimate}, the planning start point for next frame will be ahead due to the positive estimation error of the pose change, and thus making the planning to move backward under the local coordinate system. 

\subsection{Stability Limits at Sensor Offset Errors}

As discussed in Section \ref{sec:feasibility}, Prerequisite \ref{prereq:delta_V} and Prerequisite \ref{prereq:rho_bound} have to be both satisfied to ensure the stability of the local planning problem as stated in Claim \ref{clm:further_bounded}. We experimented the larger sensor offset errors to check the stability limits for the tested two scenes. 

Figure \ref{fig:scene_spd_err} shows the screenshots for the two scenes at a much larger speed measurement offset error, $v_\text{offset} = -1.0 \text{ m/s}$. Figure \ref{fig:scene_spd_err}a shows at timestamp 11.5 second, the ego car reaches to a further position compared to Figure \ref{fig:traffic_scene} and starts a third lane-change to surpass the vehicle in left lane. Figure \ref{fig:scene_spd_err}b on the other hand shows that the ego car keeps creeping forward until it crashes to the front car. This is due to the reason that the estimate error of pose change between frames is too large for the planning motion to compensate, which then gradually accumulated to larger deviations from the target pose, essentially drifting away from the target pose. Claim \ref{clm:further_bounded} would not be valid as the Prerequisite \ref{prereq:rho_bound} is not satisfied under such cases. 


\begin{figure}
    \centering
    \subfloat[Traffic scene\protect\footnotemark]{%
        \includegraphics[trim=1.2cm 0.8cm 1cm 0.5cm, clip, width=0.5\linewidth]{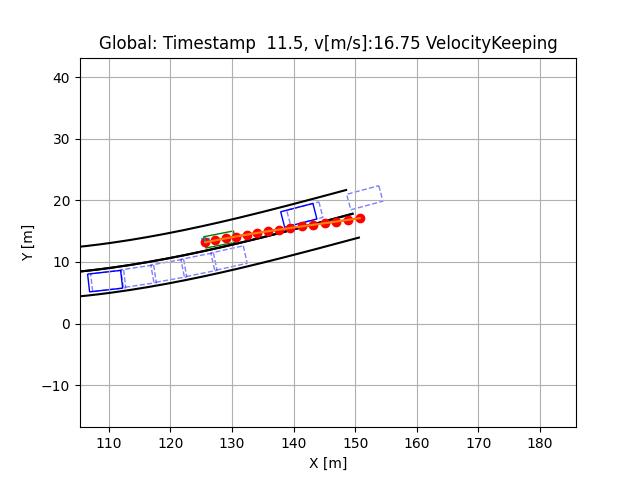}
    }
    \subfloat[Stop scene\protect\footnotemark]{%
        \includegraphics[trim=1.2cm 0.8cm 1cm 0.5cm, clip, width=0.5\linewidth]{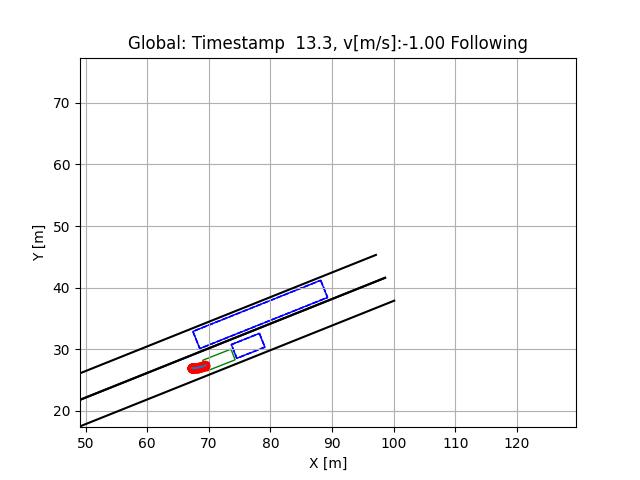}
    }
    \caption[]{Effects of speed offset error: $v_\text{offset} = -1.0 \text{ m/s}$.}
    \label{fig:scene_spd_err} 
\end{figure}
\addtocounter{footnote}{-1} 
\footnotetext{The animation gif is available at \url{https://bit.ly/3EqyTb4}.}
\addtocounter{footnote}{1} 
\footnotetext{The animation gif is available at \url{https://bit.ly/3sAavkt}.}

Figure \ref{fig:scene_yaw_1} and Figure \ref{fig:scene_yaw_2} are screenshots of the two scenes under yaw rate offset $\dot{\theta}_\text{offset}$ of 2.29 deg/s and 5.73 deg/s, respectively. At $\dot{\theta}_\text{offset}=2.29 \text{deg/s}$ in Figure \ref{fig:scene_yaw_1}, the ego car is not capable to center itself in the middle of the lane and almost rides on the right lane bound compared to Figure \ref{fig:stop_scene}. Considering that 2.29 deg/s yaw rate offset is already abnormally large, it shows that the proposed local planning framework is very robust in maintaining the stability of continous planning. Under an even larger error $\dot{\theta}_\text{offset} = 5.73  \text{deg/s}$ in Figure \ref{fig:scene_yaw_2}, the ego car drifts right so much that it eventually crosses over the lane bound and leads to potential crash or fails in both scenarios.

\begin{figure}
    \centering
    \subfloat[Traffic scene\protect\footnotemark]{%
        \includegraphics[trim=1.2cm 0.8cm 1cm 0.5cm, clip, width=0.5\linewidth]{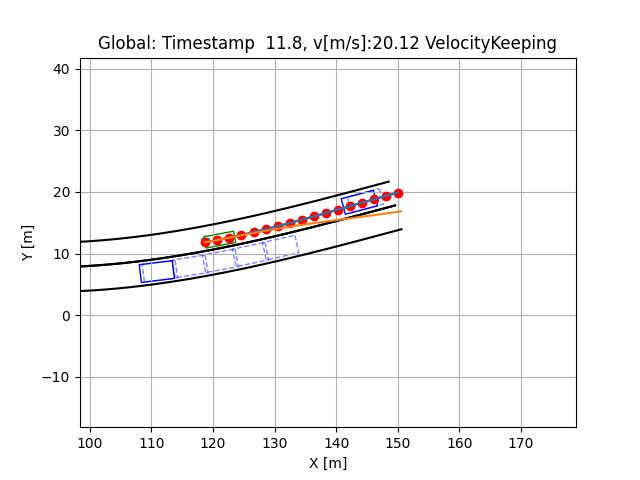}
    }
    \subfloat[Stop scene\protect\footnotemark]{%
        \includegraphics[trim=1.2cm 0.8cm 1cm 0.5cm, clip, width=0.5\linewidth]{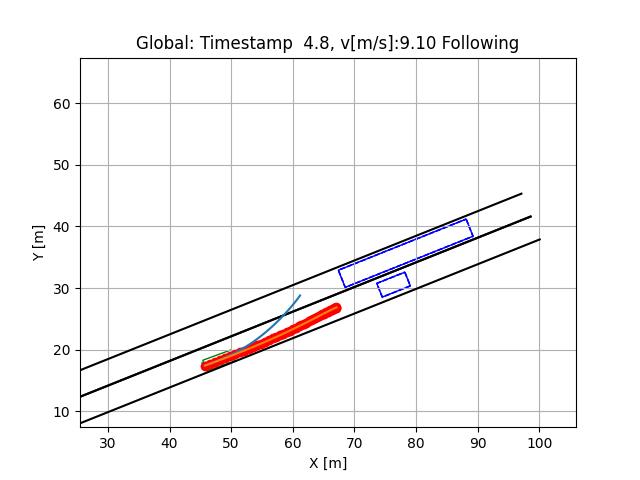}
    }
    \caption[]{Effects of yaw rate offset error: $\dot{\theta}_\text{offset} = 2.29 \text{ deg/s}$.}
    \label{fig:scene_yaw_1} 
\end{figure}
\addtocounter{footnote}{-1} 
\footnotetext{The animation gif is available at \url{https://bit.ly/3qRgSj7}.}
\addtocounter{footnote}{1} 
\footnotetext{The animation gif is available at \url{https://bit.ly/3P4hMAT}.}

\begin{figure}
    \centering
    \subfloat[Traffic scene\protect\footnotemark]{%
        \includegraphics[trim=1.2cm 0.8cm 1cm 0.5cm, clip, width=0.5\linewidth]{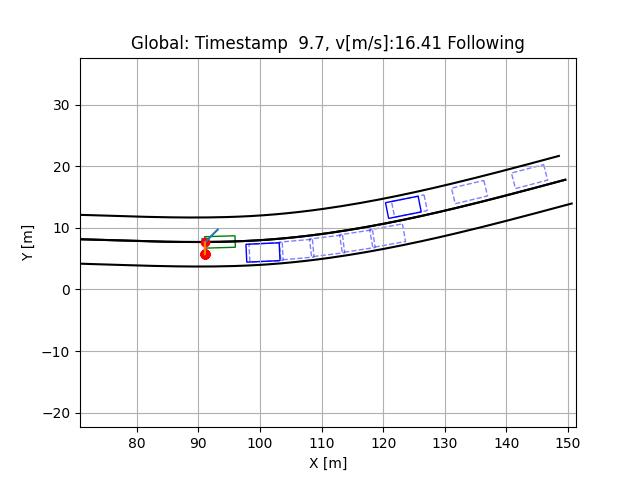}
    }
    \subfloat[Stop scene\protect\footnotemark]{%
        \includegraphics[trim=1.2cm 0.8cm 1cm 0.5cm, clip, width=0.5\linewidth]{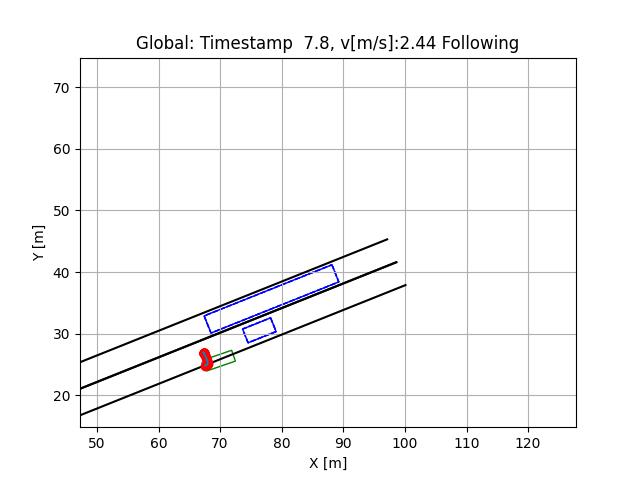}
    }
    \caption[]{Effects of yaw rate offset error: $\dot{\theta}_\text{offset} = 5.73 \text{ deg/s}$.}
    \label{fig:scene_yaw_2} 
\end{figure}


\section{Conclusions}
Overcoming the challenge of level-2+ semi-autonomous driving without relying on accurate absolute localization has been the focal point of this study. This paper has explored the viability of local trajectory planning without the need for absolute localization, focusing on a local vehicle coordinate system. The proposed local trajectory planning methodology for level-2+ semi-autonomous driving relying on the pose change estimation between planning frames from motion sensors, as well as the relative locations of traffic objects and lane lines with respect to the ego vehicle under local vehicle coordinate system. The proposed planning methodology under motion sensor errors was described as a Lyapunov stability problem with its feasibility proven under certain conditions in Section \ref{sec:feasibility}. And finally a validation pipeline was built with two scenes chosen to experiment the proposed planning framework under different error settings of speed and yaw rate measurements. The simulation results strongly support the feasibility analysis and demonstrate that the continuous planning can function properly even under relatively inferior sensor errors in \eqref{eq:v_err_val} and \eqref{eq:theta_err_val}.

\addtocounter{footnote}{-1} 
\footnotetext{The animation gif is available at \url{https://bit.ly/3ErwE7e}.}
\addtocounter{footnote}{1} 
\footnotetext{The animation gif is available at \url{https://bit.ly/44DSjUB}.}

\section*{Acknowledgments}
The authors thank the Automated Driving Lab at the Ohio State University.

\bibliographystyle{IEEEtran}
\bibliography{references}

 





\end{document}